\documentclass[letterpaper, 10 pt, conference]{ieeeconf}  
\IEEEoverridecommandlockouts                              

\usepackage{amsmath,amsfonts,amssymb,color,amsthm}
\usepackage{nccmath}

\usepackage{algorithm}
\usepackage{algpseudocode}
\usepackage{array}

\usepackage{graphicx}
\usepackage{cite}
\usepackage{footmisc}
\usepackage{mathtools}
\usepackage{bbm}
\usepackage{bm}
\usepackage{comment}
\usepackage{dsfont}

\usepackage{mathtools}
\usepackage{medmath}
\usepackage{tikz}
\usetikzlibrary{automata, positioning, arrows.meta}
\usepackage{relsize}
\usetikzlibrary{shapes,arrows}
\usepackage{cite}

\definecolor{winered}{rgb}{0.5,0,0}
\usepackage{scalerel}
\usepackage{xcolor}
\usepackage[colorlinks]{hyperref}

\newcommand{\mc}{\mathcal}

\AtBeginDocument{%
  \hypersetup{
    citecolor=blue,
    linkcolor=winered}}

\DeclarePairedDelimiter\ceil{\lceil}{\rceil}

\DeclarePairedDelimiter{\floor}{\lfloor}{\rfloor}

\newtheorem{problem}{Problem}

\newtheorem{theorem}{Theorem}

\newtheorem{lemma}{Lemma}
\newtheorem{remark}{Remark}

\title{\LARGE \bf Robust Federated Q-Learning with Almost No Communication}
\author{Sreejeet Maity and Aritra Mitra 
\thanks{The authors are with the Department of Electrical and Computer Engineering, North Carolina State University. Email: {\tt \{smaity2, amitra2\}@ncsu.edu}.}}
\begin{document}
\maketitle
\begin{abstract} We consider a federated reinforcement learning setting involving $M$ agents, all of whom interact with a common Markov Decision Process (MDP). The agents exchange information via a central server to learn the optimal value function. Our goal is to understand to what extent one can hope for collaborative sample-complexity speedups in such a setting, when a small fraction of the agents are adversarial and can act arbitrarily. To that end, we propose \texttt{\textcolor{winered}{Robust Fed-Q}}, a federated Q-learning algorithm that blends ideas from both model-based and model-free RL, along with the median-of-means device from robust statistics. We prove that despite corruption, with high-probability, \texttt{\textcolor{winered}{Robust Fed-Q}} (i) guarantees \emph{exact} convergence to the optimal value function in the limit of infinite samples, and (ii) enjoys near-optimal finite-time rates that benefit from collaboration. In addition, our approach requires just $\tilde{O}(1)$ rounds of communication to achieve each of the above guarantees, a feature of independent interest in FL where communication is the major bottleneck. 

\end{abstract}

\section{Introduction}
\label{sec:Intro}
Recent years have witnessed significant advances in the paradigm of reinforcement learning (RL), with applications spanning robotics, autonomous driving, and wireless sensor networks. In these applications, the problem of interest involves an agent (or agents) sequentially interacting with an \emph{unknown} environment with the aim of maximizing some long-term goal. The size and complexity of modern autonomous systems, such as the ones above, requires processing high-dimensional data and contending with large state and action spaces. As such, \emph{RL algorithms used in practice tend to be extremely data-hungry}, i.e., they require several data samples to achieve desired levels of accuracy. Inspired by the success of federated supervised learning~\cite{konevcny}, a natural attempt to improve accuracy is to envision a cooperative RL setting where multiple agents interacting with similar environments can exchange information to expedite the process of learning ``good" policies. This has led to the emergence of a new paradigm called federated reinforcement learning (FRL), which has shown a lot of empirical promise in reducing the sample-complexity of complex RL tasks~\cite{qiFRL}.

The hope of achieving collaborative performance gains in FRL hinges on one crucial assumption: all agents act \textit{reliably}, as expected. Such an idealistic assumption is unrealistic in large-scale systems, where certain agents can either be faulty or under attack. Blindly trusting data collected from such agents for downstream decision-making can have catastrophic consequences, especially for safety-critical applications. This leads to a \textbf{\textit{fundamental tension:}} \emph{while more data and collaboration can potentially improve performance, it can also completely disrupt the process of learning in the presence of adversarial attacks; so does more data help or hurt?}  Surprisingly, despite the surge of interest in multi-agent and federated RL, very little is understood about this fundamental tension, especially when it comes to \textit{non-asymptotic/finite-time theoretical performance guarantees}.

\textbf{The Setting.} In this context, we consider a setting involving $M$ agents, where every agent interacts with a common environment modeled as a Markov Decision Process (MDP). Like in the standard FL setting, the agents are allowed to communicate via a central server, while keeping their raw data (states, actions, and rewards) private. When all agents operate as expected, recent work~\cite{woo2023blessing} has shown that using federated variants of the celebrated model-free Q-learning algorithm~\cite{watkins1992q} can lead to provable benefits of collaboration in reducing the number of samples needed to obtain an accurate estimate of the optimal state-action value function $Q^*$. We depart from this setting by allowing a small fraction $\varepsilon$ of the agents to be \emph{worst-case adversarial}, i.e., adversarial agents are omniscient and can act arbitrarily. Our \textbf{goal} is to provide concrete answers to the following questions. 
\\
\emph{Subject to the above attack model, is it still possible to converge exactly (i.e., without any residual error) to $Q^*$? If yes, can one still hope for collaborative performance gains?}

We answer both the above questions in the affirmative for the first time by making the following contributions. 

\( \bullet \) \textbf{Algorithmic Contributions.} In Section~\ref{sec:Algorithm}, we propose a novel robust and communication-efficient federated Q-learning algorithm called \texttt{Robust Fed-Q} that blends ideas from both model-based and model-free RL. Our algorithm runs in epochs, where within each epoch, every agent uses data collected based on a synchronous (generative) sampling model~\cite{kearns, even2003learning, sidford, Waiwright, li2024q} to construct a low-variance empirical estimate of the Bellman optimality operator. The empirical Bellman operators thus constructed lead to less noisier update directions. This is a critical aspect of our approach, since less noisier update directions lower the uncertainty in the information received from uncorrupted agents. Our second key innovation is to tailor the median-of-means device~\cite{lugosiHT} from robust statistics to construct a robust aggregation scheme at the server. As we discuss in detail in Sections~\ref{sec:Algorithm} and~\ref{sec:mainresult}, each of the above aspects of our algorithm needs to be designed delicately to achieve near-optimal statistical guarantees. It is important to emphasize here that even in the absence of adversarial agents, \emph{the structure of our proposed algorithm is fundamentally different from standard FRL approaches}. In particular, the idea of constructing low-variance empirical Bellman operators by maintaining estimates of the MDP's probability transition kernels is unique to our approach. 

\( \bullet \) \textbf{Theoretical Contributions.} Our main result, namely Theorem~\ref{theorem:theoremmainresult}, provides a finite-sample guarantee on the output of \texttt{Robust Fed-Q}. When each agent has access to $T$ samples from a generative model, we establish a high-probability error bound on the order of
$$ \tilde{O}\left( \frac{1}{(1-\gamma)^{2.5} \sqrt{MT}} \right) +   \tilde{O}\left( \frac{\sqrt{\varepsilon}}{(1-\gamma)^{2.5} \sqrt{T}} \right),$$
where $\gamma \in (0,1)$ is the discount factor. When $\varepsilon=0$, i.e., there is no corruption, our bound preserves the optimal $1/\sqrt{MT}$ rate achievable with a total of $MT$ samples across agents. For sufficiently small $\varepsilon$, our algorithm continues to enjoy a benefit of collaboration. Importantly, even in the presence of corruption, our bound reveals that the final estimation error goes to $0$ in the limit of infinite samples $T$, i.e., using our approach, \emph{one can completely mitigate the effect of adversaries}. Finally, and perhaps surprisingly, we show that all of the above can be achieved with a communication overhead that is only logarithmic in both $M$ and $T$. To sum up, \textbf{we provide the first results in FRL to show that near-optimal statistical rates can be achieved despite worst-case adversarial corruption, with just $\tilde{O}(1)$ rounds of communication.} We believe this is a significant finding since robustness and communication-efficiency are both major considerations in federated RL.  

\textbf{Related Work.} We briefly discuss related work below.\\
\( \bullet \) \textbf{Single-Agent Q-learning.} While there is a rich body of work analyzing the asymptotic properties of Q-learning~\cite{borkar, tsitsiklis94}, a more recent line of literature~\cite{Qu, Waiwright, li2024q} has focused on providing finite-sample guarantees. These analyses pertain to single-agent settings and focus on the popular model-free Q-learning algorithm of Watkins~\cite{watkins1992q}. Other than the multi-agent and robustness aspects, our algorithm differs from the basic Q-learning algorithm since it interleaves model-estimation with value-function updates. 

\( \bullet \) \textbf{Federated RL.} Several recent papers~\cite{jinFRL, khodadadian, woo2023blessing, wang2023TMLR} have explored federated variants of popular RL algorithms. Our work complements these papers by considering the robustness aspect. Furthermore, as mentioned earlier, our algorithm is structurally different from typical FRL algorithms and incurs significantly less communication.

\( \bullet \) \textbf{Adversarial Robustness in Distributed Learning.} The theme of adversarial robustness has been extensively explored in distributed learning~\cite{chen, chensu, yin}, but primarily in the context of {supervised learning}/stochastic optimization. While our problem formulation is inspired by such works, the nature of our algorithms and proof techniques are fundamentally different. Among the few recent papers that have considered adversarial agents in multi-agent RL, \cite{mitra2022collaborative} and~\cite{ganesh2024global} look at bandits and policy-gradient approaches, respectively - settings that are considerably different from the tabular Q-learning formulation we consider here. Finally, while~\cite{xie2023communication} and~\cite{ye2024resilient} do consider robust multi-agent TD and Q-learning, their guarantees are asymptotic, i.e., no finite-time rates are provided in these papers.  

\section{Background and Problem Formulation}
Before describing our multi-agent setting, we first review the necessary background on Markov Decision Processes.

\textbf{MDP Model.} An MDP is denoted by $\mc{M}=(\mathcal{S}, \mathcal{A}, \mathcal{P}, R, \gamma)$, where $\mathcal{S}$ is a finite state space, $\mathcal{A}$ is a finite action space, $\mc{P}$ is a set of Markov transition kernels, $R$ is a reward function, and $\gamma \in (0,1)$ is the discount factor. When a learning agent plays action 
$a$ at state $s$, the state of the MDP transitions to $s'$ with probability $\mc{P}(s'|s,a)$, and a scalar deterministic immediate reward $R(s,a)$ is observed.\footnote{The results in this paper can be extended, with minor modifications, to account for noisy, sub-Gaussian rewards.} We assume that the rewards are bounded, i.e., $\exists \bar{R} \geq 1$ such that $|R(s,a)| \leq \bar{R}, \forall (s,a) \in \mc{S} \times \mc{A}$. We consider deterministic policies $\pi: \mathcal{S} \rightarrow \mathcal{A}$ that map states to actions. To capture the quality of a policy $\pi$, we define a $\gamma$-discounted infinite-horizon value function $V_{\pi}: \mc{S} \mapsto \mathbb{R}$ as follows:
\begin{equation}\label{eqn:V_reward}
V_{\pi}(s) = \mathbb{E}\left[\sum_{t=0}^{\infty} \gamma^t R(s_t,a_t) \, \Big| \, s_0 = s\right],
\end{equation}
where $s_t$ is the state at time $t$, $a_t = \pi(s_t)$ is the action played at time $t$, and the expectation is taken w.r.t. the randomness in the states. The basic goal in RL is to find an optimal policy $\pi^*$ that maximizes $V_\pi(s)$ simultaneously for all states $s \in \mc{S}$, \emph{without prior knowledge of the reward functions and transition kernels of the MDP}. To explain how this is done in the single-agent setting, we define the state-action value function $Q_{\pi}:\mc{S} \times \mc{A} \mapsto \mathbb{R}$ as follows: 
\begin{equation}\label{eqn:Q_reward}
Q_{\pi}(s, a) = \mathbb{E}\left[\sum_{t=0}^{\infty} \gamma^t R(s_t,a_t) \, \Big| \, (s_0, a_0) = (s, a)\right].
\end{equation}
Let $Q^{*} = Q_{\pi^*}$ denote the optimal state-action value function. Then, $Q^*$ is the unique fixed point of the Bellman optimality operator $\mathcal{T}^{*}: \mathbb{R}^{|\mc{S}|  \times |\mc{A}|}  \rightarrow \mathbb{R}^{|\mc{S}|  \times |\mc{A}|}$ given by:
\begin{equation}\label{eqn:Bellman}
    (\mathcal{T}^{*}Q)(s,a) = R(s,a) + \gamma \mathbb{E}_{s' \sim \mc{P}(\cdot | s,a)}\left[\max_{a' \in \mathcal{A}} Q(s',a')\right].
\end{equation}
In other words, $\mc{T}^* (Q^*) = Q^*.$ The Bellman operator satisfies the following contraction property $\forall Q_1, Q_2 \in \mathbb{R}^{|\mc{S}| \times |\mc{A}|}$:
\begin{equation}\label{eqn:Bellmancontraction}
\lVert \mathcal{T}^* (Q_1) - \mathcal{T}^* (Q_2)\rVert_{\infty}\le \gamma \lVert Q_1 - Q_2\rVert_{\infty}.
\end{equation}
In what follows, we briefly discuss a synchronous version~\cite{kearns, even2003learning, sidford, Waiwright, li2024q} of the celebrated Q-learning algorithm 
that exploits the above properties of $\mc{T}^{*}$ to find $Q^*$.\\
\textbf{Synchronous Single-Agent Q-learning.} The synchronous Q-learning algorithm operates in iterations $t=0, 1, \ldots$, where in each iteration $t$, the agent maintains an estimate $Q_t$ of $Q^*$. In this setting, one assumes a generative model which provides the learning agent with the following data in each iteration $t$ for every state-action pair $(s,a) \in \mc{S} \times \mc{A}$: (i) a new state $s_t(s,a)$ drawn independently from $\mc{P}(\cdot| s,a)$; and (ii) a deterministic reward $R(s,a)$. Using this information, the agent constructs an empirical Bellman operator $\mc{T}_t:\mathbb{R}^{|\mc{S}| \times |\mc{A}|} \rightarrow \mathbb{R}^{|\mc{S}| \times |\mc{A}|}$, defined as  
$$ (\mc{T}_tQ)(s,a) \triangleq R(s,a) + \gamma \max_{a'\in \mathcal{A}}Q(s_t(s,a),a'), \forall Q \in \mathbb{R}^{|\mc{S}| \times |\mc{A}|},$$
where $s_t(s,a) \sim \mc{P}(\cdot| s,a).$ Using the empirical operator $\mc{T}_t$, each component $(s,a)$ of $Q_t$ is updated as follows: 
\begin{equation}
\label{eqn:syncQ}
Q_{t+1}(s,a) = (1- \alpha_t)Q_t(s,a) + \alpha_t  (\mc{T}_tQ_t)(s,a), 
\end{equation}
where $\{\alpha_t\}$ is a suitable step-size sequence. The scheme described above is said to be \textit{synchronous} since in each iteration $t$, the agent gets to observe independent data samples for \emph{every} state-action pair, allowing every component of $Q_t$ to get updated simultaneously. Recent works~\cite{Waiwright, li2024q} have established non-asymptotic convergence rates for single-agent synchronous Q-learning, revealing that with high probability, the error $\Vert Q_T - Q^* \Vert_{\infty}$ decays as $1/\sqrt{T}$ after $T$ iterations. With this background in place, we are now ready to describe our setting of interest.

\textbf{Our Setting.} Our setting involves $M$ agents, where every agent interacts with a common environment modeled as an MDP $\mc{M}.$ To acquire information about $\mc{M}$, we assume that each agent has access to a synchronous sampling model~\cite{Waiwright, li2024q, kearns, sidford}. Furthermore, we make the standard assumption in federated RL that the data across agents are statistically independent~\cite{woo2023blessing, khodadadian, wang2023TMLR}. To be more precise, at each time-step $t$, and for each $(s,a) \in \mc{S} \times \mc{A}$, $M$ i.i.d. samples $s_{1,t}(s,a), s_{2,t}(s,a), \ldots, s_{M,t}(s,a)$ are generated from the distribution $\mc{P}(\cdot|s,a)$. The agents are allowed to communicate via a central server. However, to maintain \textbf{privacy} - a key concern in FRL - they are not allowed to exchange raw data in the form of rewards, actions, and state transitions. 

Since each agent interacts with the same MDP, it can learn $Q^*$ on its own by running the synchronous Q-learning algorithm we described earlier. So why communicate? Intuitively, if each agent can access its generative model $T$ times, then there are $MT$  total samples in the system for each state-action pair. As such, one should expect convergence to $Q^*$ at a faster rate of $\tilde{O}(1/\sqrt{MT})$, as opposed to the single-agent rate of $\tilde{O}(1/\sqrt{T})$. Recent work~\cite{woo2023blessing} has made this intuition precise and established a learning rate of $\tilde{O}(1/\sqrt{MT})$, thereby demonstrating a clear benefit of collaboration. 
\\
\textbf{Corruption Model.} We depart from the standard FRL framework by allowing a small fraction $\varepsilon \in [0, 1/2)$ of agents to be adversarial. As in robust distributed learning~\cite{yin, chensu, chen}, we consider a worst-case attack model, where the adversaries have complete knowledge of the agents' data, the MDP, and the algorithms being run. Furthermore, they can behave \emph{arbitrarily} and even collude with the goal of misleading the server and degrading global learning performance. 
\begin{problem}
\label{prob:Problem}
 Suppose each agent can access its respective generative model $T$ times. Given a confidence parameter $\delta \in (0,1)$, our goal is to develop a robust federated Q-learning algorithm that uses data from the $M$ agents to compute an estimate $\hat{Q}$ of $Q^*$ such that with probability $1-\delta$, the $\ell_{\infty}$ error $ 
 \Vert \hat{Q} - Q^* \Vert_{\infty}$ meets the following criteria: (i) decays to zero as $T \to \infty$, despite adversaries; and (ii) decays at the optimal $1/\sqrt{MT}$ rate in the absence of adversaries.
\end{problem}
In the next section, we develop a new algorithm called \texttt{Robust Fed-Q} that achieves both the requirements above with just $O(\log(MT))$ rounds of communication. 

\begin{remark} To isolate the challenges associated with robustness, we focus on a tabular RL setting under synchronous sampling. We note that to gain theoretical insights, both the generative synchronous sampling model~\cite{Waiwright, li2024q, kearns, sidford}, and the tabular setting~\cite{Waiwright, Qu, li2024q, woo2023blessing} have been extensively studied in prior RL work. Nonetheless, even for this seemingly simple setup, a complete understanding of Problem~\ref{prob:Problem} has remained open. Furthermore, as we shall see, even this setting requires the development of non-trivial algorithmic ideas. Thus, to clearly convey such ideas, we do not tackle function approximation or Markov sampling here. 
\end{remark}

\section{Algorithm} 
\label{sec:Algorithm}
\textbf{Structure of our Algorithm.} We propose an epoch-based algorithm called \texttt{Robust Fed-Q} (Algorithm~\ref{algo:Algo 2}) that interleaves the estimation of the Bellman operator with infrequent updates to the Q-table. Our approach involves $K$ epochs, each of duration $H$, such that $KH=T$, where we recall that $T$ is the total number of calls to the generative model per agent; since these calls are made in parallel across agents, $T$ can also be interpreted as the run-time duration of our algorithm. The server maintains an estimate $Q_k$ of $Q^*$ that is updated only at the end of each epoch $k \in [K]$ based on the information acquired from the $M$ agents during the latest epoch. Within each epoch, the agents perform local computations in isolation and communicate via the server at the end of the epoch. Thus, there are precisely $K$ rounds of communication. To achieve near-optimal statistical guarantees, we need to address two key questions: (i) \textbf{What} should the agents do within each epoch? (ii) \textbf{How} should the server aggregate the information received from the agents? Each of these issues needs to be dealt with delicately and, as such, requires considerable innovation. 

\(\bullet\) \textbf{Operator Refinement within Each Epoch.} Before explaining \emph{what} each agent does within each epoch, let us start with some intuition. Even in the absence of adversaries, the data available to each agent is stochastic in nature owing to the randomness in the state transitions. Thus, any object (such as a Q-table) constructed by processing such noisy data will inherit its randomness. This tells us that during the initial stages of the algorithm, \emph{the information received from different agents might appear quite different to the server because of the inherent uncertainty in our setting.} The adversarial agents can exploit this fact, making it harder for the server to distinguish between good and corrupted agents. 

To resolve the above issue, let us take a closer look at the source of randomness in a typical Q-learning update rule such as~\eqref{eqn:syncQ}. Observe that the update to the Q-table is made using an approximate version $\mc{T}_t$ of the true Bellman operator $\mc{T}^*$; moreover, $\mc{T}_t$  has high variance since it is based on just one sample, namely, the sample at time  $t$. Intuition dictates that if it were possible to make updates based on a less noisier (i.e., lower variance) estimate of $\mc{T}^*$, then the estimates of the good agents would be closer to one another. Since the good agents are in a majority (\( \varepsilon < 0.5\)), this would make it harder for the adversaries to mislead the server. 

\begin{algorithm}[t]
\caption{Median of Means Estimator \texttt{(M.o.M)}}
\label{algo:algo 1}
\begin{algorithmic}[1]
\Require Corrupted data set $\mc{X} = \{X_1, X_2, \dots, X_{{M}}\}$, corruption fraction $\varepsilon$, and confidence level $\delta$.
\State Partition $\mc{X}$ into ${P}$ disjoint buckets of size ${N} = \lfloor M/P \rfloor$,
with the $j$-th bucket denoted by $\mc{B}_j$. 
\State Compute the empirical mean for each bucket $\mc{B}_j$:
\[
\hat{\mu}_j = \frac{1}{{N}} \sum_{X_i \in \mc{\mc{B}}_j} X_i, \quad j = 1, \dots, {P}.
\]
\State Output $\tilde{\mu} = \textcolor{winered}{\texttt{Median}}\{\hat{\mu}_1, \hat{\mu}_2, \dots, \hat{\mu}_{{P}}\}.$ 
\end{algorithmic}
\end{algorithm}

Guided by the above intuition, here is our idea. Within each epoch $k \in [K]$, each good agent $i$ uses the $H$ samples acquired per state-action pair $(s,a)$ to maintain an empirical estimate of the probability transition kernel $\mc{P}(\cdot \mid s,a)$. These empirical estimates are then used to construct an empirical Bellman operator $\mc{T}_{i,k}$ with $H$-fold lower variance than the standard synchronous Q-learning algorithm. We now explain the details by fixing a good agent $i$. For each $(s,a) \in\mc{S}\times\mc{A}$, let \(\bold{1}^{(j)}(s' \mid s, a)\) denote an indicator random variable that equals 1 if a transition to state \(s'\) is observed from \((s, a)\) (i.e., if $s_t(s,a)=s'$) at the \(j^{\text{th}}\) time-step within the $k$-th epoch, and 0 otherwise.\footnote{For clarity of notation, we have suppressed the dependence of this indicator random variable on the agent index $i$, and epoch index $k$.} Using this data, agent $i$ maintains an estimate \(\hat{P}_{i,k}(s' \mid s, a) \) of $\mc{P}(s'\mid s, a)$ as follows:
\begin{equation}\label{eqn:empirical_prob}
    \hat{P}_{i,k}(s'|s,a) = \sum_{j = 1}^{H}\frac{\bold{1}^{(j)}(s'|s,a)}{H}.
\end{equation}
It is easy to see that under the synchronous sampling model, \(\hat{P}_{i,k}(s' | s, a) \) is an unbiased estimate of $\mc{P}(s' | s, a)$ with variance scaled down  by $H$. Using the estimated transition kernels, agent $i$ computes an empirical Bellman operator $\mc{T}_{i,k}: \mathbb{R}^{|\mc{S}|  \times |\mc{A}|}  \rightarrow \mathbb{R}^{|\mc{S}|  \times |\mc{A}|}$  defined as follows:
$$ 
    (\mathcal{T}_{i,k} Q)(s,a) = R(s,a) + \gamma \mathbb{E}_{s' \sim \hat{P}_{i,k}(\cdot|s,a)}\left[\max_{a' \in \mathcal{A}} Q(s',a')\right].
$$
The update direction $d_{i,k}(s,a)$ is then generated as 
\begin{equation}\label{eqn:agent-update}
    d_{i,k}(s,a) = (\mathcal{T}_{i,k} Q_k)(s,a). 
\end{equation}
For every $(s,a) \in \mc{S} \times \mc{A}$, each good agent $i$ uploads $d_{i,k}(s,a)$  to the server at the end of epoch $k$.  

\( \bullet \) \textbf{Robust Aggregation using Median-of-Means.} Now let us come to the matter of deciding how the server should aggregate the $d_{i,k}(s,a)$'s received from the $M$ agents. Two key considerations govern the choice of our robust aggregator. First, we need the output of the aggregator to concentrate tightly around the average of the uncorrupted samples (inliers) fed as input to the aggregator. Second, if each of the inliers is bounded by some finite number $B$, we would like the output to also be bounded by \emph{exactly} $B$. The second property is crucial in our analysis to ensure that the iterates generated by \texttt{Robust Fed-Q} remain uniformly bounded. Given these considerations, we tailor the median-of-means device from robust statistics to our specific needs.  

The basic Median-of-Means process is described in Algorithm~\ref{algo:algo 1}. It takes as input $M$ i.i.d. samples of a scalar real-valued random variable $X$ with mean $\mathbb{E}[X]=\mu$. A fraction $\varepsilon$ of these samples is arbitrarily corrupted. To reliably estimate $\mu$ despite outliers in the data set, the idea is to partition the $M$ samples into $P$ disjoint buckets, each containing exactly $N=\lfloor M/P \rfloor$ samples. The mean $\hat{\mu}_j$ of the samples within each bucket $j \in [P]$ is computed, and the output $\tilde{\mu}$ is the median of these means. The key design parameter is the number of buckets $P$, specified later in  Lemma~\ref{lemma:MoM}. 

\begin{algorithm}[t]
\caption{\texttt{Robust Fed-Q}}
\label{algo:Algo 2}
\begin{algorithmic}[1]
\Require Total samples $T$, confidence parameter $\delta$, and corruption fraction $\varepsilon$. 
\State Initialize $Q_0(s,a) \leftarrow 0$ for all $(s,a) \in \mc{S} \times \mc{A}.$
\For{epoch $k = 0$ to $K-1$}
\For {\((s,a) \in \mc{S} \times \mc{A}\)}
    \State Each good agent $i \in [M]$ computes \(\hat{P}_{i,k}(\cdot|s,a)\) as per \eqref{eqn:empirical_prob} and the direction $d_{i,k}(s,a)$ as per \eqref{eqn:agent-update}.
    \State All agents transmit $d_{i,k}(s,a)$ to the server.
    \State The server uses the \texttt{M.o.M.} estimator from Algorithm~\ref{algo:algo 1} with confidence parameter $\bar{\delta}$ and number of buckets $P$ chosen as per~\eqref{eqn:MoMparams} to generate $\tilde{d}_k(s,a)$ as
    \[
    \tilde{d}_k(s,a) \leftarrow \texttt{M.o.M} \left( \{d_{i,k}(s,a)\}_{i=1}^M \right). 
    \]  
    \State Server updates $Q_{k+1}(s,a)$ as per \eqref{eqn:server-update-robust}.
        \EndFor
    \State Server broadcasts $Q_{k+1}$ to all agents. 
        \EndFor
\end{algorithmic}
\end{algorithm}

We can now describe the main steps of \texttt{Robust Fed-Q} outlined in Algorithm~\ref{algo:Algo 2}. In each epoch $k \in [K]$, every  good agent $i$ computes the probability transition kernel estimates as per~\eqref{eqn:empirical_prob} and the update direction $d_{i,k}(s,a)$ as per~\eqref{eqn:agent-update}. The server applies a \texttt{M.o.M.} estimator (as in Algorithm~\ref{algo:algo 1}) to the data set $\mc{X}=\{d_{1,k}(s,a), \ldots, d_{M,k}(s,a)\}$ to compute a robust update direction $\tilde{d}_k(s,a).$ To get our desired guarantees, the confidence parameter $\bar{\delta}$ and the number of buckets $P$ for the \texttt{M.o.M.} estimator are chosen carefully as
\begin{equation}
\bar{\delta}= {\delta}/(|\mc{S}| |\mc{A}| T); \hspace{2mm} P= \ceil{8 \varepsilon M + (256/7) \log(2/\bar{\delta})},
\label{eqn:MoMparams}
\end{equation}
where $\delta \in (0,1)$ is the confidence parameter input to \texttt{Robust Fed-Q}. The above choices are informed by the analysis in Lemmas~\ref{lemma:MoM} and~\ref{lemma:drift_parameters} from Section~\ref{sec:analysis}. For each $(s,a)$, using $\tilde{d}_k(s,a)$, the server generates $Q_{k+1}(s,a)$ as follows. 
\begin{equation}\label{eqn:server-update-robust}
    Q_{k+1}(s,a) = (1 - \alpha) Q_k(s,a) + \alpha \tilde{d}_k(s,a),
\end{equation}
where $\alpha \in (0,1)$ is a constant step-size. The step-size $\alpha$ and the number of epochs $K$ will be specified later in the statement of Theorem~\ref{theorem:theoremmainresult}; see \eqref{eqn:designchoice}. For our subsequent analysis, we will assume that $M/P \geq 2.$ To meet this requirement while respecting~\eqref{eqn:MoMparams}, it suffices for $\varepsilon$ to be small enough and $M$ to be large enough such that
\begin{equation}
16 \varepsilon + \frac{512}{7M} \log(2|\mc{S}| |\mc{A}| T/{\delta}) + \frac{2}{M} < 1. 
\end{equation}
We note that similar conditions appear in~\cite{lugosi} and~\cite{yin}. 

This completes the description of our algorithm.
\begin{remark} Our approach effectively blends the model-based idea of estimating transition kernels with the model-free Q-learning rule in~\eqref{eqn:syncQ}. Furthermore, unlike standard model-free FRL algorithms where agents update their Q tables at each local step within an epoch (or round), our approach is fundamentally different in that the Q tables are never updated within an epoch; instead samples acquired during an epoch are used for estimating the Bellman operator. 
\end{remark}
\section{Main Result}
\label{sec:mainresult}
Let us define the error in the $k$-th epoch as $e_k:= \Vert Q_k - Q^* \Vert_{\infty}$. Our main result for \texttt{Robust Fed-Q} is as follows. 
\begin{theorem}\label{theorem:theoremmainresult}(\textbf{Main Convergence Result}) Given any confidence parameter $\delta \in (0,1)$, suppose the step-size $\alpha$ and the number of epochs $K$ in Algorithm~\ref{algo:Algo 2} be chosen as follows:
\begin{equation} 
\alpha = \frac{\log(MT)}{(1 - \gamma)K}; \hspace{2mm} K = \ceil{ c_1\log(MT)/(1-\gamma)},
\label{eqn:designchoice}
\end{equation}
where $c_1$ is a universal constant chosen to ensure $\alpha < 1.$ Then after $K$ epochs, the output $Q_K$ of Algorithm~\ref{algo:Algo 2} satisfies the following with probability at least \( 1 - \delta \):
\begin{equation} 
\begin{aligned}
e_K &\le \frac{e_0}{MT} +  O\left( \frac{\bar{R} \sqrt{\log(MT) \log\left(\frac{2 |\mathcal{S}||\mathcal{A}| T}{\delta}\right)}}{(1 - \gamma)^{2.5} \sqrt{MT}} \right)\\
&\hspace{1mm} + O\left(\frac{\bar{R} \sqrt{\varepsilon \log(MT)}}{(1-\gamma)^{2.5}\sqrt{T}}\right).
\end{aligned}
\label{eqn:main_conv_bnd}
\end{equation}
\end{theorem}
We defer the proof of Theorem \ref{theorem:theoremmainresult} to Section \ref{sec:analysis}.\\
\textbf{Discussion.} Theorem~\ref{theorem:theoremmainresult} provides a \emph{finite-time convergence guarantee} for robust multi-agent Q-learning in the presence of adversarial corruption. Despite adversarial corruption, the algorithm ensures that the learned Q-function \( Q_K \) remains close to the optimal Q-function \( Q^* \) with high probability. The bound on the error \( e_K \) in~\eqref{eqn:main_conv_bnd} comprises three terms: the first two terms capture the behavior of our algorithm in the absence of adversaries, and the third $O(\sqrt{\varepsilon})$ term captures the effect of adversarial corruption. In what follows, we discuss each of these terms in detail.

\(\bullet\) \textbf{Near Optimal Statistical Rates.} When there is no corruption, i.e., $\varepsilon =0$, the overall convergence rate of our algorithm is \( \tilde{O}\left({1}/{\sqrt{MT}}\right) \). When $M=1$, this rate is consistent with existing single-agent Q-learning bounds in~\cite{Qu, Waiwright}. Furthermore, our result also recovers the optimal $1/\sqrt{MT}$ guarantee in federated Q-learning with $M$ agents~\cite{woo2023blessing}, demonstrating the benefits of collaboration.

\(\bullet\) \textbf{{Vanishing Corruption Effect}.} The third term in~\eqref{eqn:main_conv_bnd} on the order of 
\( \tilde{O}\left({ \sqrt{\varepsilon}}/{\sqrt{T}}\right) \) quantifies the additional error introduced by the presence of corrupted agents, and scales with the corruption fraction \( \varepsilon \). Such an additive corruption term $\sqrt{\varepsilon}$ is typical in robust mean estimation with outliers~\cite{lugosi} and robust distributed supervised learning as well~\cite{yin, chensu}. The most distinctive feature of this term is that it diminishes with the number of samples $T$; in other words, \textbf{in the limit of infinite samples $T$, the contribution of the adversarial agents can be completely eliminated.} This is a major finding of our paper, made possible by the strategy of \emph{operator refinement} within each epoch. Doing so ensures that the error due to corruption in the Q-value update is on the order of $O(\sqrt{\varepsilon}/\sqrt{H})$ in each epoch, where $H$ is the length of the epoch;  see Lemma~\ref{lemma:drift_parameters}. Given the choice of $K$ in~\eqref{eqn:designchoice}, and the fact that $T=KH$, observe that $H$ is essentially on the order of $T$. This explains why the error due to corruption in each epoch is mitigated via a larger $T$. 
\\
\(\bullet\) \textbf{{Constant Communication}.} Finally, since the number of times the agents communicate is precisely the number of epochs $K$, observe from~\eqref{eqn:designchoice} that \texttt{Robust Fed-Q} requires just $O(\log(MT)/(1-\gamma))$ rounds of communication. Thus, \textbf{not only does} \texttt{Robust Fed-Q} \textbf{achieve near-optimal finite-time guarantees under worst-case adversaries, it does so with merely $\tilde{O}(1)$ communication rounds.}
\begin{figure}[t]
\begin{center}
\begin{tabular}{cc}
   \hspace{-6 mm}\includegraphics[scale=0.15]{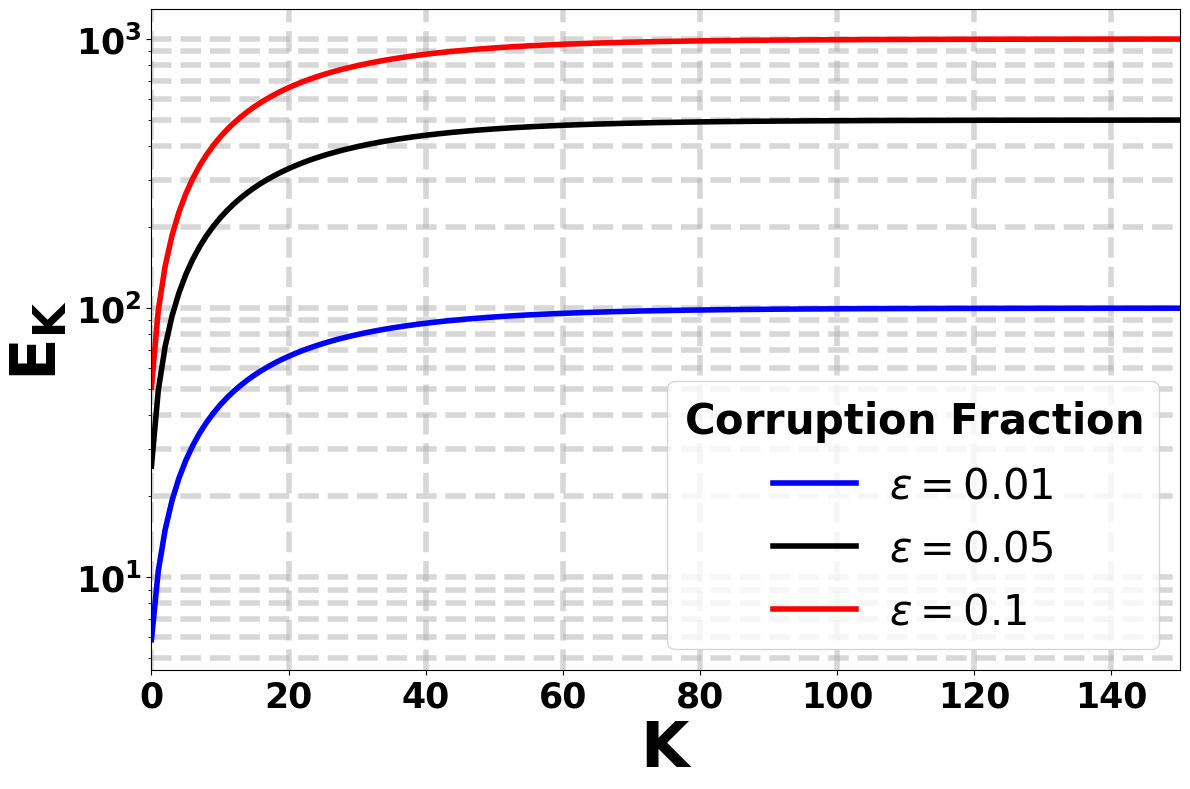}&\hspace{-4 mm}\includegraphics[scale=0.15]{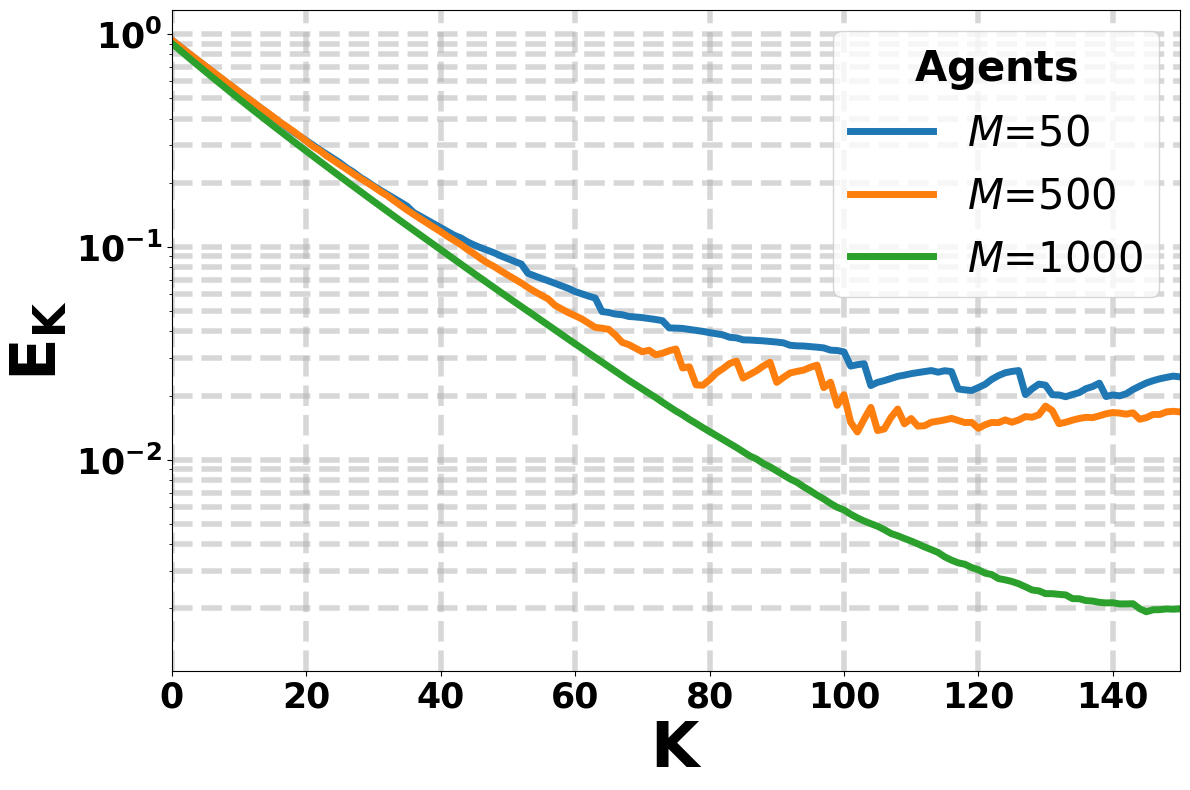}
    \end{tabular}
\vspace{-4mm}
\end{center}
\caption{\textbf{(Left)} Plots of the $\ell_\infty$ error $E_K = \lVert Q_K - Q^* \rVert_{\infty}$ for $M=1000$ and $\varepsilon \in \{0.01,0.05,0.1\}$ as a function of the number of epochs \(K\) for a vanilla federated Q-learning algorithm, where the central server simply averages the agent updates.  \textbf{(Right)} Plots of $E_K$ for \texttt{Robust Fed-Q}, with corruption fraction $\varepsilon = 0.1$ and agents $M \in \{50,500,1000\}$.}
\label{fig:sim}
\vspace{-5mm}
\end{figure}
\section{Simulation Results}
We evaluate the performance of Algorithm \ref{algo:Algo 2} on a synthetic grid-world environment with $10$ states, $5$ actions, discount factor $\gamma = 0.5$, and rewards drawn from $[0,1]$. For our simulations, every adversarial agent injects a fixed bias of $10^4$. With the step-size set to $\alpha = 0.1$, time-steps to $T = 25000$, and the confidence parameter to $\delta = 0.05$, we report our observations in Fig.~\ref{fig:sim}. Our simulations reveal that (i) a vanilla federated Q-learning algorithm that performs naive averaging can incur large errors under corruption; and (ii) \texttt{Robust Fed-Q} continues to guarantee convergence to a ball around $Q^*$, where the size of the ball reduces by increasing $M$, thus complying with our theory. 

\section{Analysis}\label{sec:analysis}
In this section, we provide a detailed finite-time analysis of our proposed algorithm \texttt{Robust Fed-Q}. To get started, we perform a simple error decomposition using the robust Q-learning update rule in~\eqref{eqn:server-update-robust}: 
\begin{equation}\label{eqn:decomp}\medmath{
\begin{aligned}
    &Q_{k+1} - Q^* = (1 - \alpha)(Q_k - Q^*) + \alpha (\tilde{d}_k - Q^*), \\
    &= (1 - \alpha)(Q_k - Q^*) + \alpha (\mathcal{T}^* Q_k - \mathcal{T}^* Q^*) + \alpha (\tilde{d}_k - \mathcal{T}^* Q_k),
\end{aligned}}
\end{equation}
where we used the fact that $\mathcal{T}^* Q^*=Q^*.$ In the absence of the third term $\alpha (\tilde{d}_k - \mathcal{T}^* Q_k)$ in the above decomposition, one can simply use contractivity of the Bellman optimality operator to complete the analysis. In our case, the bulk of the work lies in bounding $(\tilde{d}_k - \mathcal{T}^* Q_k)$ which contains both statistical errors (due to sampling), and errors due to adversarial corruption. We start with a very simple result concerning the Median-of-Means Estimator \texttt{M.o.M.} in Algorithm~\ref{algo:algo 1}. 
\begin{lemma}(\textbf{M.o.M. under Bounded Inliers})
\label{lemma:inlier_bounded}
Consider the \texttt{M.o.M.} estimation setting in Algorithm~\ref{algo:algo 1}. Suppose each uncorrupted sample $X_i$ (i.e., inlier) in the data set $\mc{X}$ satisfies $|X_i| \leq B$ for some finite $B >0.$ If $\varepsilon M < P/2$, then the output $\tilde{\mu}$ of the \texttt{M.o.M} estimator also satisfies $|\tilde{\mu}| \leq B.$
\end{lemma}
\begin{proof}
Consider any ``good" bucket $\mc{B}_j$ that contains no corrupted data samples. Since  \(\hat{\mu}_j := \frac{1}{|\mathcal{B}_j|} \sum_{i \in \mathcal{B}_j} X_i\), it is clear that if each $X_i$ within this bucket has magnitude at most $B$, then $|\hat{\mu}_j| \leq B$. Now, since at most $\varepsilon M$ samples can be corrupted, the number of such good buckets is at least $P - \varepsilon M > P/2,$ where we used $\varepsilon M < P/2$. Recall that $\tilde{\mu} = \texttt{Median}\{\hat{\mu}_1, \hat{\mu}_2, \dots, \hat{\mu}_{{P}}\}$. From the definition of the median and the fact that there are at least $P/2$ good buckets, we infer that there must exist good buckets $\mc{B}_{j_1}$ and $\mc{B}_{j_2}$ such that $- B \leq \hat{\mu}_{j_1} \leq  \tilde{\mu} \leq \hat{\mu}_{j_2} \leq B$. The claim of the lemma follows directly from the above observation. 
\end{proof}
\vspace{-1mm}
Using Lemma~\ref{lemma:inlier_bounded}, we now proceed to show that the iterates generated by \texttt{Robust Fed-Q} are uniformly bounded. 

\begin{lemma} (\textbf{Boundedness of Iterates}) 
\label{lemma:Uniformbound} The following is true for the iterates generated by Algorithm \ref{algo:Algo 2}: 
\begin{equation}
 \lvert Q_{k}(s,a) \rvert \le \frac{\bar{R}}{1-\gamma}, \forall (s,a) \in \mathcal{S} \times \mathcal{A}, \forall k \geq 0,
 \label{eqn:Q_bound}
\end{equation}
where recall that $|R(s,a)| \leq \bar{R}, \forall (s,a) \in \mc{S} \times \mc{A}.$
\end{lemma}
\begin{proof}
We will prove this result via induction. Since \(\bar{R} \ge 1\), and \(Q_0(s,a) = 0, \forall (s,a) \in \mc{S} \times \mc{A}\) in Algorithm \ref{algo:Algo 2}, \eqref{eqn:Q_bound} holds trivially for \( k=0 \). Now suppose the bound in~\eqref{eqn:Q_bound} holds for all epochs up to epoch $k$. We need to show that the same bound applies to $Q_{k+1}(s,a).$ To that end, fix a state-action pair $(s,a) \in \mc{S} \times \mc{A},$ and let us recall how $Q_{k+1}(s,a)$ is generated. In epoch $k$, each good agent $i$ generates $d_{i,k}(s,a)$ as per~\eqref{eqn:agent-update}. The server then constructs a robust estimate $\tilde{d}_k(s,a)$ by applying a \texttt{M.o.M.}  estimator to the data set $\mc{X} = \{d_{1,k}(s,a), d_{2,k}(s,a), \ldots, d_{M,k}(s,a)\},$ and $Q_{k+1}(s,a)$ is subsequently updated using $\tilde{d}_k(s,a)$ as per~\eqref{eqn:server-update-robust}. For each good agent $i \in [M]$, we have from \eqref{eqn:agent-update}:
\begin{equation}
\begin{aligned}
    \lvert d_{i,k}(s,a) \rvert &\leq \lvert R(s,a) \rvert + \gamma \mathbb{E}_{s' \sim \hat{P}_{i,k}(\cdot|s,a)} \lvert \max_{a' \in \mathcal{A}} Q_k(s',a') \rvert \\
    &\leq \bar{R} + \gamma \frac{\bar{R}}{1-\gamma} = \frac{\bar{R}}{1-\gamma}, \nonumber
\end{aligned}
\end{equation}
where in the second step, we used the induction hypothesis and the fact that $|R(s,a)| \leq \bar{R}.$ We conclude that each inlier in the data set $\mc{X} = \{d_{1,k}(s,a), d_{2,k}(s,a), \ldots, d_{M,k}(s,a)\}$ has magnitude at most $\bar{R}/(1-\gamma).$ Furthermore, from the definition of the number of buckets $P$ in Algorithm~\ref{algo:Algo 2}, we have that $P > 2 \varepsilon M.$ Invoking Lemma~\ref{lemma:inlier_bounded} then tells us that $|\tilde{d}_k(s,a)| \leq \bar{R}/(1-\gamma).$ From~\eqref{eqn:server-update-robust}, we then have 
\begin{equation}
\begin{aligned}
    &\lvert Q_{k+1}(s,a) \rvert \leq (1-\alpha)\lvert Q_k(s,a) \rvert + \alpha \lvert \tilde{d}_k(s,a) \rvert, \\
    &\leq (1-\alpha) \frac{\bar{R}}{1-\gamma} + \alpha \frac{\bar{R}}{1-\gamma} = \frac{\bar{R}}{1-\gamma},
    \end{aligned}
\end{equation}
where we once again used the induction hypothesis. This completes the induction claim. 
\end{proof}
Next, we establish high-probability concentration bounds for the \texttt{M.o.M.} estimator in Algorithm~\ref{algo:algo 1} by carefully exploiting properties of sub-Gaussian random variables.\footnote{A random variable $X \in \mathbb{R}$ is said to be sub-Gaussian with variance proxy $\sigma^2$ (or $\sigma$-sub-Gaussian) if its moment-generating function satisfies $\mathbb{E}[\exp(sX)] \leq \exp(s^2 \sigma^2/2), \forall s\in \mathbb{R}$~\cite{rigollet2023high}.} 

\begin{lemma}(\textbf{High-Probability Guarantees for M.o.M. under Adversarial Contamination})\label{lemma:MoM} Consider the M.o.M. estimation setting in Algorithm~\ref{algo:algo 1}, where the corrupted data set \(\mc{X} \triangleq \{ X_1, X_2, \dots, X_{{M}}\}\) comprises $M$ i.i.d. samples of a scalar random variable $X$, of which, at most $\varepsilon M$ samples are arbitrarily corrupted. Let $\mu= \mathbb{E}[X]$, and suppose each uncorrupted sample $i$ is such that $(X_i - \mu)$ is $B$-sub-Gaussian for some finite $B >0.$ Given any $\delta \in (0,1)$, suppose the number of buckets $P$ in the M.o.M. estimator be chosen as follows: $P = \ceil{8 \varepsilon M + (256/7) \log(2/\delta)}.$
Then, the output $\tilde{\mu}$ of the \texttt{M.o.M.} procedure in Algorithm~\ref{algo:algo 1} satisfies the following bound with probability at least $1-\delta$:  
    \begin{equation}
       \lvert \tilde{\mu} - \mu \rvert \leq c \cdot B  \left(\sqrt{\frac{\log(2/\delta)}{{M}}} + \sqrt{\varepsilon}\right),
    \end{equation}
    where \( c > 1 \) is a universal constant.
\end{lemma}  
\begin{proof}
To provide a guarantee on \(\tilde{\mu}\), we recall from the description of Algorithm~\ref{algo:algo 1} that the data set $\mc{X}$ is partitioned into $P$ disjoint buckets, with each bucket containing $N=\lfloor {M}/{P} \rfloor$ samples. We first aim to establish probabilistic guarantees for the buckets that do not contain any corrupted samples (referred to as \emph{good buckets}). Let the set of all such buckets be denoted by \(\mathcal{B}^g\). Now, for a particular bucket \(\mathcal{B}_j \in \mathcal{B}^g\) with no contamination, we define the following:
\[
Y \triangleq \hat{\mu}_j - \mu = \left( \frac{1}{|\mc{B}_j|}\sum_{i \in \mathcal{B}_j} X_i - \mu \right) = \frac{1}{|\mc{B}_j|} \sum_{i \in \mathcal{B}_j} \underbrace{( X_i - \mu)}_{Y_i}.
\]
Due to the i.i.d. nature of the samples, we have \(\mathbb{E}[X_i] = \mu\), $\forall i \in \mc{B}_j$,  implying \(\mathbb{E}[Y_i] = 0\), where $Y_i := X_i - \mu$. Furthermore, by assumption, $Y_i$ is a $B$-sub-Gaussian random variable. Since \(Y\) is the average of $N$ i.i.d. sub-Gaussian random variables, each with variance proxy $B^2$, it follows that \(Y\) is itself sub-Gaussian with variance proxy $B^2/N$ \cite[Lemma 5.4]{csabaasymp}. 
Hence, for any $\lambda >0$, we can apply Markov's inequality to bound the probability \( \mathbb{P}(|Y| \ge \lambda)\) as
\begin{equation}\label{eqn:confidence}
\begin{aligned}
    \mathbb{P}(|Y| \ge \lambda) 
    = \mathbb{P}(|Y|^2 \ge \lambda^2) 
    \le \frac{\mathbb{E}[|Y|^2]}{\lambda^2} 
    \le \frac{4{B}^2}{{N} \lambda^2},
\end{aligned}
\end{equation}
where we used the fact that \(Y\) is sub-Gaussian with variance proxy $B^2/N$, implying \(\mathbb{E}[Y^2] \le 4B^2/N\) using \cite[Lemma 1.4]{rigollet2023high}. Now, setting the R.H.S. of \eqref{eqn:confidence} to a desired confidence level $1/4$, we get $\lambda = 4B/\sqrt{N}$. Next, note that $N = \floor{M/P} \ge M/P - 1 \overset{(*)}{\ge} M/(2P)$, 
where for \((*)\), we used $M \geq 2P.$ Using the above fact along with $\lambda = 4B/\sqrt{N}$ in~\eqref{eqn:confidence}, we obtain the following bound for the sample mean $\hat{\mu}_j$ of the \emph{good} bucket $\mc{B}_j$: 
\begin{equation}
    \mathbb{P} \left( \Big\lvert \mu - \hat{\mu}_j \Big\rvert \ge 4 B\sqrt{\frac{2{P}}{{M}}} \right) \le \frac{1}{4}.
\label{eqn:good_bucket_prob}
\end{equation}

To translate the above guarantees to the median $\tilde{\mu}$, 
we define an indicator random variable \(\mathcal{Z}_j\) for each uncontaminated bucket \(\mathcal{B}_j \in \mathcal{B}^g\), where
\[
\mathcal{Z}_j = 
\begin{cases}
1 & \text{if } \hat{\mu}_j 
\geq \mu + 4B\sqrt{\frac{2{P}}{{M}}}, \\
0 & \text{otherwise}. 
\end{cases}
\]
We then have: 
\begin{equation}\label{eqn:transition_med_to_mean}
    \mathbb{P}\left(\tilde{\mu} \ge \mu + 4B\sqrt{\frac{2{P}}{{M}}}\right) \overset{(*)}{\le} \mathbb{P} \left(\sum_{j \in \mc{B}^g} \mc{Z}_j \ge {P}/2 -\varepsilon {M}\right). 
\end{equation}
We now justify the inequality \((*)\) in~\eqref{eqn:transition_med_to_mean}. First, note that the condition $\tilde{\mu} \ge \mu + 4B\sqrt{2P/M}$ implies that at least \({P}/2\) buckets have sample means greater than \(\mu + 4B\sqrt{2P/M}\). Given that at most \(\varepsilon {M}\) samples are corrupted, the number of corrupted buckets is at most \(\varepsilon M\). Therefore, there must be at least \(({P}/2 - \varepsilon M)\) uncontaminated (i.e., \emph{good}) buckets whose means exceed \(\mu + 4B\sqrt{2P/M}\), justifying~\eqref{eqn:transition_med_to_mean}. Next, let $\tilde{N}$ denote the number of good buckets, and observe
\begin{equation}\label{eqn:pre-final_mom}
\begin{aligned}
&\left(\sum_{j \in \mc{B}^g} \mc{Z}_j \ge {P}/2 -\varepsilon M \right) \\
&\overset{(a)}\implies
 \left(\frac{1}{\tilde{N}}\sum_{j \in \mc{B}^g}\mc{Z}_j - \mathbb{E}[\mc{Z}_j] \ge \frac{{P}/2 -\varepsilon M}{\tilde{N}}-\frac{1}{4} \right)\\
 &\overset{(b)}\implies
 \left(\frac{1}{\tilde{N}}\sum_{j \in \mc{B}^g}\mc{Z}_j - \mathbb{E}[\mc{Z}_j] \ge \frac{1}{4} - \frac{\varepsilon {M}}{P} \right)\\
 &\overset{(c)}\implies
 \left(\frac{1}{\tilde{N}}\sum_{j \in \mc{B}^g}\mc{Z}_j - \mathbb{E}[\mc{Z}_j] \ge \frac{1}{8}\right).\\
\end{aligned}
\end{equation}
In the above steps, for (a), we used the definition of $\mc{Z}_j$ and~\eqref{eqn:good_bucket_prob} to infer that $\mathbb{E}[\mc{Z}_j] \leq 1/4.$ For (b), we used $\tilde{N} \leq P$, and for (c), we picked the number of buckets $P$ to satisfy $P \geq 8 \varepsilon M$, implying $\tilde{N} \geq P - \varepsilon M \geq (7/8) P$. Since the $0-1$ indicator random variables $\mc{Z}_j, j \in \mc{B}^g$, are independent, we can use Hoeffding's inequality along with equations~\eqref{eqn:transition_med_to_mean},~\eqref{eqn:pre-final_mom}, and $\tilde{N} \geq (7/8) P$ to obtain
\begin{equation}
\begin{aligned}
& \mathbb{P}\left(\tilde{\mu} \ge \mu + 4B\sqrt{\frac{2{P}}{{M}}}\right) \leq \mathbb{P}\left(\frac{1}{\tilde{N}}\sum_{j \in \mc{B}^g}\mc{Z}_j - \mathbb{E}[\mc{Z}_j] \ge \frac{1}{8}\right) \\
&\leq \exp(-\tilde{N}/32)  \overset{(a)}\leq \exp(-7P/256) \leq \delta/2,
\end{aligned}
\end{equation}
provided $P$ is chosen to satisfy $P \geq (256/7) \log(2/\delta).$ Since we also require $P$ to satisfy $P \geq 8 \varepsilon M$, the number of buckets can be chosen as follows: 
 $P = \ceil{8 \varepsilon M + (256/7) \log(2/\delta)            }.$
Under the above choice of $P$, we have just shown that with probability at least $1-\delta/2$, 
$$ \tilde{\mu}-\mu \leq 4B\sqrt{\frac{2{P}}{{M}}} \leq c \cdot B  \left(\sqrt{\frac{\log(2/\delta)}{{M}}} + \sqrt{\varepsilon}\right)$$
for some suitably large universal constant $c$. Using an identical analysis, one can establish a lower bound on $\tilde{\mu} - \mu$ that also holds with probability $1-\delta/2.$ Union-bounding completes the proof. 
\end{proof}

We now proceed to bound $\tilde{d}_k - \mc{T}^*Q_k$. 

\begin{lemma}(\textbf{Bounding Adversarial Effects})
    The following bound holds with probability at least \(1 - \delta\), for all \( k \in [K]\): $  \left\lVert \tilde{d}_k - \mathcal{T}^{*}Q_k \right\rVert_{\infty} \leq W,$ where 
    \begin{equation}
         W:= \frac{2c \bar{R}}{(1-\gamma)\sqrt{H}} \left(\sqrt{\varepsilon} + \sqrt{\frac{\log(2 |\mathcal{S}||\mathcal{A}|T/\delta)}{M}}\right),
\label{eqn:Wbnd}
    \end{equation}
and $c$ is the universal constant from Lemma~\ref{lemma:MoM}.     
    \label{lemma:drift_parameters}
\end{lemma}
\begin{proof}
The proof strategy is to extend Lemma~\ref{lemma:MoM} to our setting. To that end, fix a state-action pair \((s,a) \in \mathcal{S} \times \mathcal{A}\) and an epoch $k \in [K]$. Recall from Algorithm~\ref{algo:Algo 2} that $\tilde{d}_k(s,a) = \texttt{M.o.M} \left( \{d_{i,k}(s,a)\}_{i=1}^M \right),$ where for each good agent $i$, $d_{i,k}(s,a)$ is updated as per~\eqref{eqn:agent-update}. The next immediate step is to understand the statistics of $d_{i,k}(s,a)$ for a particular good agent $i$. Accordingly, recall from~\eqref{eqn:empirical_prob} that \(\hat{P}_{i,k}(\cdot|s,a)\) is an empirical estimate of the transition kernel \(\mc{P}(\cdot|s,a)\) computed by agent \(i\) at epoch \(k\) over an epoch length of $H$ time-steps. We can then re-write \(d_{i,k}(s,a)\) as follows:
\begin{equation}
\resizebox{0.95\hsize}{!}{$
    d_{i,k}(s,a) = \frac{1}{H} \sum_{j=1}^{H} \underbrace{\left(R(s,a) + \gamma \sum_{s' \in \mc{S}}\bold{1}^{(j)}(s' \mid s,a) \max_{a' \in \mathcal{A}} Q_k(s',a') \right)}_{X_j},$}
\nonumber
\end{equation}
where $\bold{1}^{(j)}(s' \mid s,a)$ is an indicator random variable capturing whether state $s'$ is observed from pair $(s,a)$ at the $j$-th time-step within the $k$-th epoch. Let \(\mathcal{F}_{k-1}\) denote the \(\sigma\)-algebra generated by all the randomness up to the end of epoch $k-1$. We then have
\begin{equation} \medmath{
\begin{aligned}
    \mathbb{E}[X_j|\mc{F}_{k-1}] & \overset{(\bullet)}{=} R(s,a) + \gamma \sum_{s' \in \mc{S}} \mathbb{E}[\bold{1}^{(j)}(s' \mid s,a)] \max_{a' \in \mathcal{A}} Q_k(s',a'),\\
    & \overset{(\bullet \bullet)}{=} R(s,a) + \gamma \sum_{s' \in \mc{S}} \mc{P}(s' \mid s,a) \max_{a' \in \mathcal{A}} Q_k(s',a'),\\
    & = \mc{T}^{*}Q_k(s,a).
    \end{aligned}}
\end{equation}
In \((\bullet)\), we used the fact that \(Q_k\) is \(\mathcal{F}_{k-1}\)-measurable, and in \((\bullet\bullet)\), we used  \(\mathbb{E}\left[\bold{1}^{(j)}(s' \mid s,a)\right] = \mc{P}(s' \mid s,a)\). Thus, conditioned on $\mc{F}_{k-1}$, $X_j - \mc{T}^* Q_k(s,a)$ is zero-mean, $\forall j \in [H].$ Next, we show that for each $j \in [H]$, \(X_j - \mc{T}^{*}Q_k(s,a)\) is bounded deterministically:
\begin{equation}
        \left| X_j - \mathcal{T}^{*}Q_k(s,a) \right| 
        \le \left| X_j \right| + \left| \mathcal{T}^{*}Q_k(s,a)\right|
        \leq \frac{2 \bar{R}}{(1-\gamma)} :=B. 
\nonumber 
\end{equation}
To see why the above is true, recall from Lemma~\ref{lemma:Uniformbound} that we have shown $|Q_k(s,a)| \leq \bar{R}/(1-\gamma)$. Using the fact that $|R(s,a)| \leq \bar{R}$, it is then easy to see that the same upper-bound applies to both $X_j$ and $\mc{T}^*Q_k(s,a)$. Hence, conditioned on \(\mc{F}_{k-1}\), \(\{X_j - \mathcal{T}^{*}Q_k(s,a)\}_{j=1}^{H}\) is an i.i.d. $B$-sub-Gaussian sequence with $B=2\bar{R}/(1-\gamma)$~\cite[Example 5.6]{lattimore2020bandit}, where the i.i.d. aspect follows from synchronous sampling. Since $d_{i,k}(s,a) =1/(H) \sum_{j \in [H]} X_j,$ we then conclude that conditioned on $\mc{F}_{k-1}$, $d_{i,k}(s,a)-\mc{T}^{*} Q_k(s,a)$ is itself sub-Gaussian with variance proxy \(B^2/H\). Finally, conditioned on $\mc{F}_{k-1}$, notice that the only randomness left in $d_{i,k}(s,a)$ comes from the state transitions during the $H$-length epoch, which are assumed to be independent across agents. Thus, conditioned on $\mc{F}_{k-1}$, the inliers in the data set $\{d_{1,k}(s,a), \ldots, d_{M,k}(s,a)\}$ are independent. With the choice of $P$ in~\eqref{eqn:MoMparams}, we can now directly appeal to Lemma~\ref{lemma:MoM} to conclude that conditioned on $\mc{F}_{k-1}$, the following event $\mc{E}$ occurs with probability at least $1-\bar{\delta}$: 
\begin{equation}\label{eqn:pre-final-error-bound} \medmath{
\mc{E} \triangleq \Bigg\{\left|\tilde{d}_k(s,a) - \mc{T}^{*}Q_k(s,a)\right| \leq \frac{cB}{\sqrt{H}} \left(\sqrt{\varepsilon} + \sqrt{\frac{\log(2/\bar{\delta})}{M}}\right)}\Bigg\},
\nonumber
\end{equation}
where $c$ is the universal constant from Lemma~\ref{lemma:MoM}. Letting $\bold{1}_{\mc{E}}$ be the indicator of the event $\mc{E}$, we further have $\mathbb{P}(\mc{E}) = \mathbb{E}[\bold{1}_{\mc{E}}] = \mathbb{E}[\mathbb{E}[\bold{1}_{\mc{E}}|\mc{F}_{k-1}]] \ge 1-\bar{\delta},$ where we used $\mathbb{P}(\mc{E}|\mc{F}_{k-1}) \geq 1-\bar{\delta}.$ Union-bounding over all state-action pairs \((s,a) \in \mc{S} \times \mc{A}\), epochs \(k \in [K]\), and using $K \leq T$, we conclude that the following bound holds \emph{simultaneously} $\forall (s,a) \in \mc{S} \times \mc{A}$ and $\forall k \in [K]$ with probability at least $1-\delta$: 
\begin{equation} 
    \left|\tilde{d}_k(s,a) - \mc{T}^{*}Q_k(s,a)\right|
    \leq \frac{cB}{\sqrt{H}} \left(\sqrt{\varepsilon} + \sqrt{\frac{\log(2 |\mathcal{S}||\mathcal{A}|T/\delta)}{M}}\right),
    \nonumber
\end{equation}
where we set $\bar{\delta}=\delta/(|\mc{S}| |\mc{A}| T).$ The fact that the same bound as above applies to $\left\lVert \tilde{d}_k - \mathcal{T}^{*}Q_k \right\rVert_{\infty}$ follows from the definition of the infinity norm. 
\end{proof}
We are now ready to complete the proof of Theorem~\ref{theorem:theoremmainresult}.

\begin{proof}\textbf{\textit{(Proof of Theorem \ref{theorem:theoremmainresult})}} Let us start by defining $e_k:= \Vert Q_k - Q^*\Vert_{\infty}$. Taking the \(\infty\)-norm on both sides of the error decomposition in~\eqref{eqn:decomp}, and using the contractive property of the Bellman operator in~\eqref{eqn:Bellmancontraction}, we obtain 
\begin{equation}
    e_{k+1} \le (1 - \alpha(1 - \gamma)) e_k + \alpha \| \tilde{d}_k - \mathcal{T}^* Q_k \|_\infty.
\end{equation}
Iterating this bound over \( K \) epochs yields the following:
\begin{equation} \medmath{
\begin{aligned}
    e_K &\le \underbrace{(1 - \alpha(1 - \gamma))^K e_0}_{(*)} \\
    & + \underbrace{\sum_{r=0}^{K-1} \alpha (1 - \alpha(1 - \gamma))^{K-1 - r} \| \tilde{d}_r - \mathcal{T}^* Q_r \|_\infty}_{(**)}. 
\end{aligned}}
\label{eqn:penult_bnd}
\end{equation}
Lemma~\ref{lemma:drift_parameters} informs us that there exists a ``good event" $\mc{G}$ with measure at least $1-\delta$, on which, \(\lVert \tilde{d}_r - \mathcal{T}^* Q_r \rVert_{\infty} \le W, \forall r\), where $W$ is as in~\eqref{eqn:Wbnd}. On event $\mc{G}$, \((**)\) can be bounded as 
$$ (**)  \le \alpha {W} \sum_{p=0}^{\infty} (1 - \alpha(1 - \gamma))^{p} = \frac{{W}}{1-\gamma}.$$
Plugging the above bound into~\eqref{eqn:penult_bnd}, in the event $\mc{G}$ we have 
\begin{equation}
e_K \leq (1 - \alpha(1 - \gamma))^K e_0 + \frac{W}{(1-\gamma)}.
\label{eqn:final_bnd}
\end{equation}
To further refine the above bound and arrive at the final form in~\eqref{eqn:main_conv_bnd}, we choose the step size $\alpha$ and the number of epochs $K$ as per~\eqref{eqn:designchoice}, and use $T=KH$.
\end{proof}

\section{Conclusion}
We considered a collaborative RL problem, and developed a novel robust federated Q-learning algorithm that enjoys near-optimal statistical gains from collaboration, despite the presence of adversarial agents. As future work, we plan to derive lower bounds for our setting, consider Markov sampling, and  function approximation. 
\bibliographystyle{IEEEtran} 
\bibliography{refs}
\end{document}